\documentclass[nohyperref]{article}

\usepackage{microtype}
\usepackage{graphicx}
\usepackage{subfigure}
\usepackage{booktabs} 
\usepackage{tabularx}
\usepackage{stfloats}
\usepackage{colortbl}
\usepackage{siunitx}
\usepackage{hyperref}

\usepackage[nohyperref,accepted]{icml2025}

\usepackage{amsmath}
\usepackage{amssymb}
\usepackage{mathtools}
\usepackage{amsthm}
\usepackage{subcaption}

\usepackage[capitalize,noabbrev]{cleveref}

\theoremstyle{plain}
\newtheorem{theorem}{Theorem}[section]

\newtheorem{lemma}[theorem]{Lemma}

\theoremstyle{definition}

\newtheorem{condition}[theorem]{Condition}
\theoremstyle{remark}

\usepackage[textsize=tiny]{todonotes}

\icmltitlerunning{}

\def\T{\top}
\def\Diag{\mathrm{Diag}}
\def\sf{\mathrm{sf}}

\begin{document}

\onecolumn
\icmltitle{A mean teacher algorithm for unlearning of language models}




\begin{icmlauthorlist}
\icmlauthor{Yegor Klochkov}{bd}
\end{icmlauthorlist}

\icmlaffiliation{bd}{Work done at ByteDance Research}

\icmlcorrespondingauthor{Yegor Klochkov}{eklochov@gmail.com}

\icmlkeywords{Machine Learning, ICML}

\vskip 0.3in



\printAffiliations{}  

\begin{abstract}
One of the goals of language model unlearning is to reduce memorization of selected text instances while retaining the model's general abilities. Despite various proposed methods, reducing memorization of large datasets without noticeable degradation in model utility remains challenging. In this paper, we investigate the mean teacher algorithm \cite{tarvainen2017mean}, a simple proximal optimization method from continual learning literature that gradually modifies the teacher model. We show that the mean teacher can approximate a trajectory of a slow natural gradient descent (NGD), which inherently seeks low-curvature updates that are less likely to degrade the model utility. While slow NGD can suffer from vanishing gradients, we introduce a new unlearning loss called ``negative log-unlikelihood'' (NLUL) that avoids this problem. We show that the combination of mean teacher and NLUL improves some metrics on the MUSE benchmarks \cite{shi2024muse}.
\end{abstract}

\section{Introduction}

In the recent years, Large Language Models (LLMs) have reached unprecedented capabilities that are achieved through training on massive datasets. Often comprising hundreds of billions of tokens, these datasets are collected from diverse internet sources. This data-intensive approach raises significant ethical concerns as some protected user and copyright data may be compromised. The General Data Protection Regulation (GDPR) establishes fundamental rights for individuals, including the ``right to be forgotten,'' which grants the deletion of personal data upon request. Concurrently, the inclusion of copyrighted content in training datasets has initiated legal quarrels \cite{doe1_github_2022, tremblay_openai_2023}. {These regulatory and legal actions have sparked interest in methods  that would  reduce the influence of  selective  training data on predictions of large pretrained models, a process known as machine unlearning \cite{bourtoule2021machine, ginart2019making}.}

In unlearning of language models, we are aimed at modifying the model weights that would make them behave as if it was never trained on particular selected data, which we refer to as \emph{forget set} \cite{yao2023large,eldan2023s,cooper2024machine,shi2024muse}. For example, we want to suppress certain outputs, like being able to reproduce some memorized training instances word-by-word with completion requests, or answering questions about their content. Following \cite{shi2024muse}, we refer to the completion ability as \emph{verbatim memorization}, and the ability to answer questions as \emph{knowledge memorization}. We therefore want to reduce such memorization, ideally, without decreasing the model's utility.

There have been a lot of interest in the recent literature towards language model unlearning, with various methods proposed. A generally accepted approach consists of finetuning a model by optimizing an objective specifically designed to discourage memorization. Much of research is dedicated to designing specific loss functions that encourage unlearning, with notable examples such as log-likelihood (referred to as gradient ascent) \cite{thudi2022unrolling,yao2023large}, Negative Preference Optimization \cite{zhang2024negative}, and others \cite{chundawat2023can,fan2024simplicity,wang2024llm,li2024wmdp}. With the resulting objective including a regularization term, it is typically optimized with AdamW \cite{loshchilov2017fixing}, which is a standard way of optimizing language models. A few recent works dedicated to unlearning address the question of how we optimize the corresponding objective, \citet{wang2025gru} employ some gradient projections, and \citet{bu2024unlearning} consider adaptive learning rates designed specifically for unlearning.

We revisit a simple method called mean teacher \cite{tarvainen2017mean}, which so far appeared only in the context of continual learning literature. We show that in certain scenarios, it can approximate the trajectory of a slow natural gradient descent. This means that the algorithm performs updates along the low curvature directions, which can help retain the model utility and avoid neural collapse. However, we show that mean teacher can suffer from vanishing gradients with some popular choices of unlearning losses, such as log-likelihood and NPO. Instead, we introduce a very simple loss called Negative Log UnLikelihood (NLUL) that avoids this problem and combines well with mean teacher.


For the experiments, we focus on the MUSE benchmark \cite{shi2024muse}, which contains rather large forget set (3.3M tokens). Although there are methods that are capable of removing verbatim memorization, the reported results so far show it is challenging to reduce knowledge memorization without dramatic reduction in the utility. We show that one of the versions of mean teacher can achieve this. However, we show that in this case the reduction of knowledge memorization is accompanied by reduction of the Massive Multitask Language Understanding metric (MMLU) \cite{hendrycks2020measuring}. This means that the evaluations in the MUSE-benchmark may not be enough for adequate assessment of unlearning algorithms that researchers state reduce knowledge memorization. On the other hand, there are variants that achieve competititve verbatim memorization reduction and utility preservation, while improving on another metrics that are part of the MUSE benchmark associated with risks of privacy leakage.

In our work, we focus exclusively on methods that edit the model weights. We mention that some methods are based on expanding the model decision space through guardrails \citep{thaker2024guardrail, liu2024large} and training low-rank adapters on top of the pretrained model weights \citep{gao2024practical, ji2024reversing}. Unlearning is sometimes seen as a safety mechanism, with focus on reducing potentially harmful knowledge of language models \cite{li2024wmdp}. There, the focus is on removing knowledge of whole topics or concepts. On the other hand, we focus on unlearning independent text instances.

\section{Unlearning set-up and description of the method}

Most unlearning benchmarks assume that we are given an access to two datasets:
\begin{itemize}
    \item \emph{forget set} $D_{f}$ is a set of training instances that we know are the source of knowledge we want to remove;
    \item \emph{retain set} $D_{r}$ is another set of training instances that do not contain undesirable knowledge, which is intended to assist in retaining the model's utility.
\end{itemize}
We acknowledge knowing the source of undesirable knowledge would already be a strong assumption. However, most of the existing benchmarks follow such set up \cite{maini2024tofu, shi2024muse}, see also a review in \cite{thaker2024position}.

Often, unlearning is performed by optimizing an objective of the form,
\begin{equation}\label{unlearning_objective}
    L_{forget}(D_{f}; \theta) + L_{retain}(D_{r}; \theta),
\end{equation}
where $L_{forget}$ averages some unlearning loss on the given examples. That is,
\begin{equation}\label{forumla_for_loss_how}
    L_{forget}(D_{f}; \theta) = \frac{1}{|D_{f}|} \sum_{(x, y) \in D_{f}} \ell(h(x;\theta), y),
\end{equation}
where $h(x; \theta)$ are the logit outputs of the model on the given instance $x$, and $y$ is the target label. One popular choice of unlearning loss, is the \emph{log likelihood} (LL)
\[
    \ell_{LL}(h, y) = \log \sf(h)_y, 
\]
where $\sf(h)_y = e^{h_y} / \left(\sum_j e^{h_j} \right)$ is the softmax function. Such loss is a straightforward negation of the standard learning loss---negative log likelihood (NLL)---and it is intended to discourage correct prediction on the inputs from the forget set. Other unlearning losses include \emph{Negative Preference Optimization} (NPO) \cite{zhang2024negative}; The \emph{Incompetent Teacher} loss outputs KL divergence between the outputs $h(x; \theta)$ and that of a model of smaller size that does not exhibit memorization \cite{chundawat2023can}. In Section~\ref{section_unlearning_losses},  we additionally introduce a new loss called \emph{negative log-unlikelihood}. While NPO is designed to address the problem with explosive gradients when unlearning with LL, our new loss also addresses vanishing gradients at the start of the unlearning process. In Section~\ref{section_unlearning_losses}, we elaborate on this issue in detail, and we also provide some empirical arguments in Section~\ref{section_experiments}.

The penalizing term $L_{retain}$ is designed to maintain the utility of the model, which is usually defined as either NLL loss on the instances from $D_{r}$, or the Kullback-Leibler (KL) divergence with the outputs of the initial model. To be precise, in the latter case the retain loss looks as follows,
\begin{equation}\label{kl_loss}
    L_{retain}(D_r; \theta) =
    \frac{1}{|D_{r}|} \sum_{(x, y) \in D_{r}} KL(\sf(h(x; \theta)) \; \| \; \sf(h(x; \theta_0)))
\end{equation}
The experiments often suggest that KL regularization leads to better results \cite{shi2024muse, yao2023large}.

Everywhere in the paper we assume that the datasets contain  next tokens as prediction labels, i.e., ${D}_{f}$ is formed of pairs $(x,y)$ --- context and next token --- where $ x = (s_{0}, \dots, s_{t-1})$, $y = s_{t}$, and $s$ is a pretraining sequence with $t < |s|$.  Similarly, $D_{r}$ consists of context and next token pairs for sequences from the retain set.

\section{Mean teacher algorithm and it's approximation of natural gradient descent}

When the retain loss as in \eqref{kl_loss} is used, the optimization problem  \eqref{unlearning_objective} can be seen as a proximal optimization problem
\[
    \min_{\theta} L(\theta) + \mathcal{D}(\theta, \theta'),
\]
where we have $L(\theta) = L_{forget}(\theta)$, the reference model $\theta' = \theta_{0}$ is equal to the starting point, and the divergence term is the KL loss on the retain set
\begin{equation}\label{divergence_kl}
    \mathcal{D}(\theta, \theta') = \frac{1}{|D_{r}|} \sum_{(x, y) \in D_{r}} KL(\sf(h(x; \theta)) \; \| \; \sf(h(x; \theta'))).
\end{equation}
The divergence term measures the proximity of the model $\theta$ to the reference $\theta'$. For example, it is non-negative, and in fact, in the case of KL divergence it is also locally quadratic, in the sense that there is a matrix $H(\theta)$, such that
\begin{equation}\label{local_quadratic}
    \mathcal{D}(\theta, \theta') = \frac{1}{2} (\theta - \theta')^{\T} H(\theta') (\theta - \theta') + o(\| \theta - \theta'\|^{2})
\end{equation}
In the case of KL divergence, this matrix corresponds to the Gauss-Newton Hessian (GNH),
\begin{equation}\label{gnh_definition}
    H(\theta) = \frac{1}{|D_{r}|} \sum_{x \in D_{r}} [J_{\theta} h(x; \theta)] S_{h(x; \theta)} [J_{\theta} h(x; \theta)]^{\T},
\end{equation}
where $ J_{\theta} h(\theta) = \left( \frac{\partial h_{j}(\theta)}{\partial \theta_{i}} \right)_{ij} $ is the Jacobian, and $ S_{h} = \Diag(\sf(h)) - \sf(h)\sf(h)^{\T}$. The Gauss-Newton Hessian is semi-positive definite by design, and it is often regarded as an approximation to the Hessian of the NLL loss \cite{martens2020new}. We also note that the formula above is only a particular case of the GNH corresponding to the cross entropy loss, and it can be defined for other losses as well. Furthermore, in this particular case, it is equivalent to the Fisher Information Matrix \cite{schraudolph2002fast, kunstner2019limitations}.

Let us now describe our version of mean teacher algorithm. For the sake of generality, let us assume that we have some loss function $L(\theta)$, that is not necessarily an unlearning loss, and some divergence $\mathcal{D}(\theta, \theta')$ that is not necessarily the KL divergence. The mean teacher works in gradient steps and maintains the model we are optimizing $\theta_{t}$ along with the reference model $\theta_t'$ that is also updated with every iteration as exponentially weighted average of the past models,
\[
     \theta_{t}' = (\kappa\eta) \theta_{t} + (\kappa\eta) (1 - \kappa\eta) \theta_{t - 1} + (\kappa\eta) (1 - \kappa\eta)^{2} \theta_{t -2} + \dots ,
\]
The two models are initialized with the given starting model $\theta_{0}' = \theta_0$, followed by the updates
\begin{enumerate}
    \item $\theta_{t + 1} \leftarrow \theta_t - \eta \nabla \{ \alpha L(\theta_t) + \mathcal{D}(\theta_t, \theta_t')\}$;
    \item $\theta_{t + 1}' \leftarrow (1 - \eta\kappa) \theta_{t}' + \eta\kappa \theta_{t+1}$,
\end{enumerate}
where $\eta $ is the learning rate, and $\kappa$ is the contraction hyperparameter, and we additionally introduce a weight $\alpha < 1$ for the loss term. 
The first step is a gradient descent step on the regularized loss, while the second updates the reference model by slightly sliding it towards the optimized model. Choosing a small weight $\alpha$ can ensure that the optimized model $\theta_{t}$ will stay close to the reference model $\theta_{t}'$, while the sliding reference update ensures that we can still progress away from the initial point $\theta_0$. 

Suppose that the given divergence $\mathcal{D}(\theta, \theta')$ satisfies \eqref{local_quadratic} with some $H(\theta)$. With some abuse of terms, we refer to this matrix as \emph{Hessian}. In addition, we consider adding a quadratic regularization to improve the condition number of the Hessian
\begin{equation}\label{hessian_damping}
    \mathcal{D}_{\lambda}(\theta, \theta'):=  \mathcal{D}(\theta, \theta') + \frac{\lambda}{2} \| \theta - \theta'\|^{2},
    \qquad
    H_{\lambda}(\theta) := H(\theta) + \lambda I.
\end{equation}
Furthermore, we additionally apply momentum accumulation to the gradients \cite{polyak1964some, nemirovskij1983problem}. The pseudo code of the resulting algorithm is summarized in Algorithm~\ref{algo_trust}.

Below we show that for a sufficiently small coefficient $\alpha$, the trajectory of Algorithm~\ref{algo_trust} follows a gradient descent, where the matrix $H(\theta)$ serves as a conditioner. In the case, where $H(\theta)$ is GNH, such trajectory corresponds to \emph{natural gradient descent} \cite{amari1998natural}.

\begin{algorithm}[tb]
   \caption{Mean teacher algorithm}
   \label{algo_trust}
\begin{algorithmic}
   \STATE {\bfseries Input:} Learning rate $\eta$, {damping parameter $\lambda$, momentum $\mu$,} weight $\alpha$, number of steps $T$, contraction parameter $\kappa$, starting parameter~$\theta_0$
   \STATE Initialize reference $\theta_{0}' := \theta_0$.
   \FOR{$t=1$ {\bfseries to} $T$}
   \STATE $\theta_{t} \leftarrow \theta_{t-1} - \eta \nabla_{\theta} \left\{ \alpha L(\theta_{t-1})+\mathcal{D}_{{{\lambda}}}(\theta_{t-1}, \theta_{t-1}')\right\} + \mu(\theta_{t-1} - \theta_{t-2})$
   \STATE $\theta_{t}' \leftarrow (1 - \eta \kappa) \theta_{t-1}' + \eta \kappa \theta_{t}$
   \ENDFOR
   \STATE {\bfseries Return:} $\theta_{T}$
\end{algorithmic}
\end{algorithm}

\begin{theorem}\label{NG_mathcing_theorem}
Set $\gamma := \kappa \alpha \eta / (1 - \kappa \eta)$ and {$\overline{\lambda} := \lambda + (1 - \mu) \kappa / (1 - \eta \kappa)$}. Consider the following updates
\begin{equation}\label{ng_formulation}
    \overline{\theta}_{t + 1} = \overline{\theta}_{t} - \gamma H_{\overline{\lambda}}(\overline{\theta}_{t})^{-1} \nabla L(\overline{\theta}_{t-1}) \,,
    \qquad \overline{\theta}_{0} = \theta_0
\end{equation}
Suppose that 1) $\kappa$ is a positive constant, 2) $\eta$, $\alpha$ are sufficiently small, 3) number of steps $T$ is such that $T \gamma $ is bounded by a constant, 4) $H(\theta)$ is symmetric positive-definite and satisfies Eq.~\eqref{local_quadratic} 5) $\mathcal{D}(\theta, \theta')$, $L(\theta)$, and $H(\theta)$ satisfy some regularity conditions, which we postpone to the appendix (Condition~\ref{regularity_condition}). Then, the trajectory of Algorithm~\ref{algo_trust} approximately matches that of \eqref{ng_formulation}, with the following upper-bound
\[
    \max_{t \leq T} \| \theta_{t} - \overline{\theta}_t \| \leq O(\alpha \log(1/\alpha))\,.
\]
\end{theorem}

Note, that the contraction coefficient $\kappa$ effectively increases the damping parameter. This effect can be reduced with the use of momentum. For example, for a typical choice $\mu= 0.9$, we get $\overline{\lambda} \approx \lambda + 0.1 \kappa $. We additionally note that the condition on $T\gamma$ bounds the total sum of step sizes in \eqref{ng_formulation}. The proof is deferred to Section~\ref{ng_proof} in the appendix.

\subsection*{Alternative divergences}

The KL divergence is the most common way to regularize in proximal optimization methods. We additionally consider the {\bf Quadratic KL}(QKL) loss defined as
\begin{align}
    \mathcal{D}(\theta, \theta') &= \frac{1}{|D_{r}|} \sum_{(x, y) \in D_{r}} QKL(h(x; \theta), h(x; \theta')) \label{qkl_def} \\
    QKL(h, h') &:= (h - h')^{\T} (\Diag(\sf(h)) - \sf(h)\sf(h)^{\T})(h - h'), \nonumber
\end{align}
which attains the same approximation as the KL divergence in Eq.~\eqref{local_quadratic}, with GNH in place of $H(\theta)$.
We show in the experiments section that using QKL can sometimes give different results and improve some of the metrics in the benchmarks.

For the sake of example, we also mention that $H(\theta)$ does not always have to be the GNH. If we have a loss function $L(\theta)$ that is convex, we can consider the associated Bregman divergence,
\[
    \mathcal{D}(\theta, \theta') = L(\theta) - L(\theta') - (\theta - \theta')^{\T}\nabla L(\theta'),
\]
which due to Taylor expansion allows quadratic approximation with the original Hessian
\(H(\theta) = \nabla^{2} L(\theta).\) Such divergence was is in \cite{amid2022public} for a private mirror descent algorithm.

\subsection*{Using pretraining data for calculating divergence}

\citet{tarvainen2017mean} introduce mean teacher as a semi-supervised finetuning method with the focus on classification models. The idea is to use a limited labeled set for the loss, and a large set of unlabeled examples for the divergence, which does not depend on the labels. In language model finetuning, all text examples are self-supervised since we always predict the next token. However, the examples we select still undergo some curation. For example, in MUSE-Books the forget contains Harry Potter books, while the retain split contains articles from Harry Potter Fandom Wiki pages. If we are faced with a new unlearning task, we will have to obtain curate a new retain split that matches the topic of a forget set in a similar way. Furthermore, in the unlearning benchmarks, the retain split is often not much larger and sometimes smaller than the forget set. In order to make the set up closer to the semi-supervised setting, we propose to replace the retain split with a subset of pretraining data. For language models, we propose to use the OpenWebText-2, which is readily available on Huggingface\footnote{\url{https://huggingface.co/datasets/Skylion007/openwebtext}}. This dataset has moderate size by modern standards and often included in larger pretraining corpora \cite{gao2020pile}. The dataset comprises of $<$1B tokens and our algorithms only use a fraction of that by sampling an independent batch with each gradient update. Apart from the advantage of having more tokens than a benchmark's retain split, it can also be reused when we are faced with a new unlearning task which reduces the cost of collecting and curating the data.

We additionally note that using pretraining subset $D_{pt}$ instead of $D_{r}$ in Eq.~\eqref{gnh_definition} may be better aligned with preserving general abilities of the model. \citet{ghorbani2019investigation} numerically demonstrated that the subspaces spanned by top  eigenvectors of GNH are aligned with that of Hessian of the cross entropy loss \(H(\theta) = \nabla^{2} L_{CE}(\theta)\). The natural gradients reduce the influence of these directions due to inversion of the Hessian, which therefore ensures that our updates have low curvature w.r.t. the pretraining loss
\[
    L_{CE}(\theta; D_{pt}) = \frac{1}{|D_{pt}|} \sum_{(x, y) \in D_{pt}} \ell_{CE}(h(x, \theta), y).
\]
This can help better preserve the general abilities of the language model that are not directly associated with the unlearning task at hand.
We therefore replace $D_r$ with the pretraining subset $D_{pt}$ in the experiments for two reasons: 1) it is a bigger set than $D_r$ and it does not depend on the unlearning task at hand 2) it can help us better preserve general abilities of the model.

\subsection*{Implementation details}
Finally, notice that Theorem~\ref{NG_mathcing_theorem} assumes full-batch gradient updates. In the experiments, we make the following changes according to standard practice:
\begin{enumerate}
    \item \emph{Batching:} The gradients are calculated on batches of data, i.e., for each update we sample a batch of sequences from $D_{f}$ and a batch $D_{pt}$. The loss and the KL divergence are calculated over that batch.
    \item \emph{Gradient clipping:} we apply norm clipping to the gradients to avoid explosive gradients. Note, since that clipping is equivalent to reducing the learning rate value, we correspondingly reduce the contracting parameter $\kappa \leftarrow \kappa / \max(\|g\| / c, 1)$, where $c$ is the clipping coefficient.
    \item \emph{Momentum:} the gradients are fed into the momentum SGD optimizer.
\end{enumerate}
For reference, we provide the pseudo-code for this version of the algorithm in the appendix, {}{Section~\ref{section_batched}.}


\section{Unlearning losses}\label{section_unlearning_losses}
Here we describe three popular unlearning losses, in addition to one which we introduce in this paper. The given objectives are to be minimized.

\begin{enumerate}
    \item \emph{Log likelihood (LL)} is a straightforward negation of the standard learning loss --- negative log likelihood (NLL),
    \[
        \ell_{LL}(h; y) = \log \sf(h)_y,
    \]
    which is also sometimes referred to as gradient ascent, since it negates the gradients that would be used in learning with NLL.
    Using this loss for unlearning has two caveats: firstly, we typically start from a point where the we are pretty much converged on $D_{forget}$ due to memorization, and it is hard to escape this point, unless we are using an accelerated optimization method such as AdamW. However, this does not solve the second problem: we will observe  explosive gradients the further away we get from the minima, which can eventually lead to a collapse of the model.
    \item \emph{Negative preference optimization (NPO)} is a modification of direct preference optimization loss (DPO), that only includes the negative examples, which are the ones sampled from the forget set,
    \[
        \ell_{NPO_{\beta}}(s; \theta) = -\frac{2}{\beta} \log \sigma \left( - {\beta} \log \frac{\pi_{\theta}(s)}{\pi_{base}(s)} \right)
    \]
    \cite{zhang2024negative} introduce this loss to address the problem with explosive gradients. 
    In particular, they observe that the gradients of this loss correspond to reweighed gradients of the log-likelihood,
    \begin{equation*}
        \frac{1}{|s|} \nabla_\theta \ell_{NPO_{\beta}}(s; \theta)
        =\frac{\pi(s)^{\beta}}{\pi_{\theta}(s)^{\beta} + \pi_{base}(s)^{\beta}}  \frac{1}{|s|} \sum_{(x, y) \in s} \nabla_{\theta} \ell_{LL}(h(x, \theta); y) \,.
    \end{equation*}
    Here, the weight will approach zero, as the probabilities $\pi_{\theta}(s) $ on the forget set reduce. The term $\pi_{base}(s)$ is the probability of the forget token according to the starting model, which we denote as ``base'' to avoid confusion with the reference model that we update according to the mean teacher rule. This means that we have to additionally store and forward propagate through one more model, which moderately increases the cost of unlearning. To avoid confusion, we remark that this loss does not conform with the formula Eq.~\ref{forumla_for_loss_how} as the the weight is calculated per sequence.

    \item \emph{Incompetent teacher} (IT) is a simple and intriguing approach that was introduced in \cite{chundawat2023can}, but appears have been overlooked in the LLM unlearning literature. They propose to minimize the KL divergence between outputs of the optimized model and that of a fixed smaller model that potentially does not exhibit memorization --- the incompetent teacher. For example, if we unlearn a LLama-2 7B model, we can ``push'' its outputs towards that of the TinyLlama~1.1B on the forget data, rather than maximizing the NLL. The loss in that case looks as follows,
    \[
        \ell_{IT}(h; y) = KL\left( \sf(h) \; \| \; \sf(h_{it}) \right)
    \]
    where $h_{it}$ are the logit outputs of the incompetent teacher. KL divergence, being convex w.r.t. the logits $h$, is better suited to combine with popular accelerated optimization techniques, such as AdamW. Notice that the mismatch loss from \cite{yao2023large} is equivalent to incompetent teacher choosing uniformly random tokens.
    
    \item \emph{Negative log-unlikelihood (NLUL)} is another candidate that we introduce in this paper, and it looks as follows
    \[
        \ell_{NLUL}(h; y) = - \log (1 - \sf(h)_y)
    \]
    The idea is similar to the LL approach, but instead of minimizing the likelihood, we maximize the log-unlikelihood $\log (1 - p(y))$. Similar to NPO, such reformulation allows a better reweighing of the gradients per token
    \[
        \nabla_h \ell_{NLUL}(h; y) = \frac{\sf(h)_y}{1 - \sf(h)_y} \nabla_{h} \ell_{LL}(h; y)\,.
    \]
    It has a similar effect when $p(y) \approx 0$, helping to avoid explosive gradients. However, it additionally allows to escape the stagnating starting point by making more aggressive updates in the beginning, when we often have $p(y) \approx 1$ in the forget set. It does not introduce additional hyperparameters (such as $\beta$ in NPO) and does not require to maintain an additional base model. In the experiments, we find this loss combines best with MT.
\end{enumerate}

\section{Experiments}\label{section_experiments}

\subsection{Benchmarks}

We use two benchmarks from \cite{shi2024muse}, MUSE-News and MUSE-Books. In both cases, the authors provide us with a \emph{forget set}, a \emph{retain set}, a set of queries and ground truth labels for testing memorization and utility, and a \emph{target model} to be unlearned. MUSE-News is based on BBC news articles collected after the release date of Llama-2. These articles are split into forget and retrain sets randomly. Then, they finetune a Llama-2 7B on both sets, so that it memorizes both sets of articles. The goal is to reduce memorization of the forget data while retaining knowledge of the retain articles. For this, they provide a set of completion and question queries extracted from the news articles, which provides the measurement of both retain and forget articles. Overall, the following metrics are considered:
\begin{itemize}
    \item {\tt verbmem\_f}: this metric measures verbatim memorization of the forget data -- the ability to reproduce part of the article when prompted with the completion request. For this, they provide a 100 pairs prompt-completion, and measure the ROUGE-L score between the response to a prompt and ground truth completion.
    \item {\tt knowmem\_f}: this metric evaluates knowledge memorization of the forget data-- the ability to answer question about some facts contained in the article. For this, \citet{shi2024muse} provide 100 pairs of questions and ground truth answers extracted from the forget split articles. Then, we measure how similar model's response to the ground truth answer. They propose to use ROUGE-L to measure this similarity.
    \item {\tt knowmem\_r}: similarly to knowledge memorization of the forget data, this metric measures that of the retain data, with QA pairs extracted from the retain articles. This value is proposed as a measure of utility of the model.
    \item {\it PrivLeak}: they additionally propose a metric of how easy it is to detect that the forget set was used for unlearning with some state-of-the-art membership inference methods for language models \cite{shi2023detecting}. For exact definition, see Section~{3.1} in \cite{shi2024muse}, and we use the code they provide for evaluations. This metric can have both positive and negative values, and in the ideal situation it is close to zero.
\end{itemize}
In addition to that, we measure the MMLU accuracy \cite{hendrycks2020measuring}, which is a popular metric of language model's general knowledge. While original set of questions is composed of approximately 150K questions, to speed up the evaluations, we use a small validation set consisting of $\approx 1500$ questions. We denote the corresponding accuracy as {\tt mmlu\_val}.

In the MUSE-Books benchmark, the forget set contains four Harry Potter books, while the retain set consists of random articles from the Harry Potter FanWiki page. They similarly construct a set of queries, allowing one to measure {\tt verbmem\_f}, {\tt knowmem\_f}, {\tt knowmem\_r}, and we additionally include {\tt mmlu\_val}. The target model that was finetuned on both forget and retain splits, and it is a version of Llama-2 architecture as well.

In both cases, we discard the provided retain split and instead rely on the pretraining dataset OpenWebText.

\subsection{Methods}

For baselines, we optimize unlearning loss penalized by KL divergence on $D_{pt}$ with AdamW optimizer. The methods are referred to as AdamW + (LL/NPO/IT/NLUL) + KL.

The proposed methods utilize mean teacher, with one of the four unlearning losses, and either KL or QKL regularization on $D_{pt}$. The mean teacher algorithm is implemented in accordance with Algorithm~\ref{algo_trust_batched} in the appendix. The corresponding methods are labeled as MT + (LL/NPO/IT/NLUL) + (KL/QKL).

For the IT loss, we use TinyLlama-1.1B\footnote{https://huggingface.co/TinyLlama/TinyLlama-1.1B-Chat-v1.0}, which shares the tokenizer with Llama-2.
{We list all hyperparameters for each experiment in Section~\ref{section_hyper} in the appendix.}

\subsection{Mean teacher can suffer from vanishing gradients}\label{loss_ablation}

\begin{figure}[t]
    \begin{minipage}[t]{0.48\textwidth}
        \centering
        \includegraphics[width=0.9\textwidth]{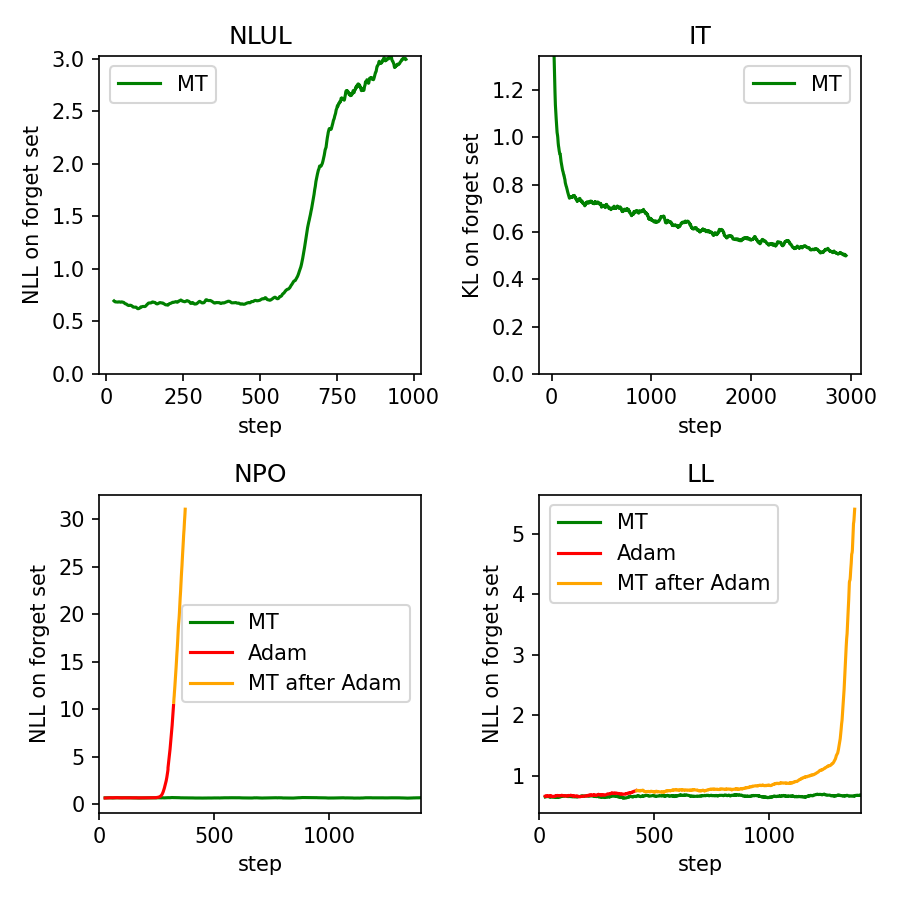}
        \caption{NLL loss on the forget set for MT using different unlearning losses (green). For IT we only show the KL divergence between the bad teacher and the target model. For NPO/LL we additionally perform 2 epochs with AdamW to  ``escape'' the starting point.}
        \label{fig:loss_dynamics}
    \end{minipage}
    \hfill
    \begin{minipage}[t]{0.48\textwidth}
        \centering
        \includegraphics[width=0.95\textwidth]{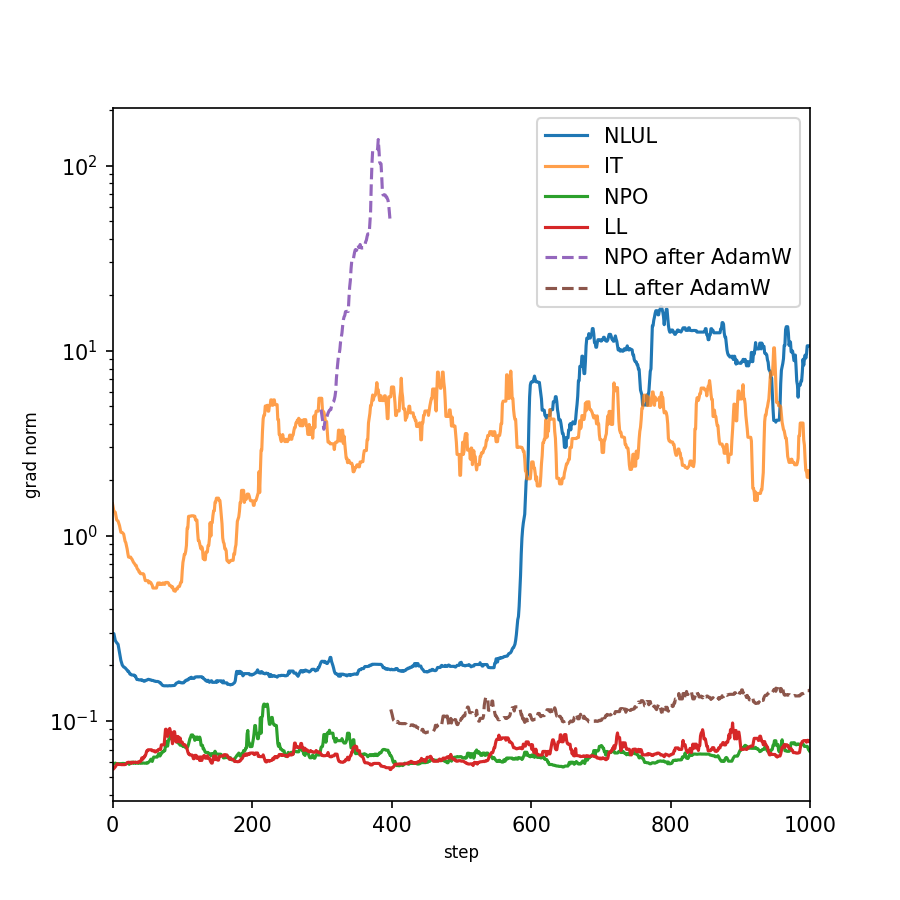}
        \caption{Gradient norms during MT training in Figure~\ref{fig:loss_dynamics}.}
        \label{fig:grad_dynamics}
    \end{minipage}
\end{figure}

We first perform experiments with each of the four losses to find out which combines best with MT. We use MUSE-News data and report the training loss in 
Figure~\ref{fig:loss_dynamics}. In the cases of LL, NPO, NLUL, we report the NLL loss on the forget set. In the case of IT we report KL divergence with the logit outputs of TinyLlama-1.1B. Observe that in the latter case the algorithm does not manage to reduce memorization even after 3000 steps, despite the fact that the KL loss keeps reducing. Furthermore, when used with NPO and LL losses, the method struggles to escape the initial point, see the green line on the bottom plots in 
Figure~\ref{fig:loss_dynamics}. To show that this is only a problem in the starting point, we perform a 2 epoch finetuning with AdamW, and once we escaped that starting point with low loss, MT manages to continue increasing NLL. On the contrary, NLUL manages to increase NLL without any additional tricks. We show the gradient norm per step in 
Figure~\ref{fig:grad_dynamics}, which also shows lower order gradients magnitudes for LL and NPO than NLUL. We conclude that despite a plausible natural gradient approximation, MT only approximates slow natural gradient descent and can suffer from vanishing gradients. Out of these four unlearning losses, NLUL combines best with MT out of these four choices, and we will use the MT + NLUL combination in the remaining experiments.

\begin{table*}[t]
\centering
\caption{{Results for mean teacher and baselines on MUSE-News/Books. We run each method 3 times and report average and std of each metric, and we also report number of epochs. The  ($\star$) sign indicates that the method is run until {\tt knowmem\_f}$\leq 31.1$. {\it PrivLeak} is only measured on one run.}}\label{results_all}

\begin{tabular}{l|llll|l}
\toprule
 & verbmem\_f $\downarrow$ & knowmem\_f $\downarrow$& knowmem\_r $\uparrow$ & {\it PrivLeak} & mmlu\_val $\uparrow$ \\
\midrule
\multicolumn{6}{c}{\cellcolor{gray!20}{MUSE-News}} \\
Target & 56.3 & 64.2 & 55.2 & -99.8 & 40.0 \\
Retrain & 21.8 &  33.1 & 55.0 & \phantom{-1}0.0 & NA \\
\midrule
MT + NLUL + KL & 14.8 (2.3) & 58.2 (0.5) & {\bf 52.7 (0.9)} & -38.8 & 39.0 (0.5) \\
MT + NLUL + QKL & \phantom{0}8.4 (0.7) & {\bf 40.0 (6.2)} & 50.8 (0.5) & \phantom{1}{\bf -6.9} & 31.0 (4.3) \\
AdamW + NPO + KL & \phantom{0}0.0 (0.0) & 51.1 (1.3) & {50.1 (0.9)} & \phantom{-}89.0 & 39.1 (0.8) \\
AdamW + IT + KL & 21.3 (0.2) & 64.4 (0.1) & {\bf 52.7 (0.2)} & -99.5 & {\bf 40.5 (0.1)} \\
AdamW + LL + KL & \phantom{0}0.0 (0.0) & 51.2 (1.2) & 50.7 (0.5) & \phantom{-}91.9 & 38.0 (0.0) \\
AdamW + NLUL + KL & \phantom{0}2.9 (0.7) & 58.4 (0.3) & 49.9 (0.9) & \phantom{-}94.2 & 39.6 (0.2) \\
\midrule
MT + NLUL + QKL${\,\!}^\star$ & \phantom{0}8.5 (0.9) & {\bf 24.8 (3.1)} & 51.3 (0.2) & \phantom{-}25.0 & 19.5 (0.6) \\
\midrule
\multicolumn{6}{c}{\cellcolor{gray!20}MUSE-Books} \\
Target & 99.8 & 59.4 & 66.9 & -57.5 & 26.3 \\
Retrain & 14.3 & 28.9 & 74.5 & \phantom{-1}0.0 & NA \\
\midrule
NLUL + MT + KL & 1.5 (0.3) & 22.0 (3.8) & {\bf 59.1 (0.7)} & -35.9 & {\bf 26.6 (0.2)}\\
NLUL + MT + QKL & 8.8 (1.0) & {\bf 17.5 (1.5)} & 43.0 (3.2) &  -57.7& 25.6 (0.1)\\
AdamW + NPO + KL & 0.0 (0.0) & {\bf 17.7 (6.8)}  & 37.17 (5.3) & {\bf -24.1} & 25.8 (0.4) \\
\bottomrule
\end{tabular}
\end{table*}

\begin{table*}[t]
\centering
\caption{Examples of answers from target model of the MUSE-News benchmark and it's unlearned modifications. We show one example probing verbatim memorization (completion) and one example probing knowledge memorization (QA). Examples are cherry-picked for illustration purpose.}\label{tab_examples}
\begin{tabular}{>{\raggedright\arraybackslash}p{2.7cm}>{\raggedright\arraybackslash}p{7.9cm}>{\raggedright\arraybackslash}p{3.5cm}}
\toprule
Prompt / Question & {\tiny [\dots] The BMA wants a 35\% pay rise to make up for 15 years of below-inflation wage rises, It says falling pay has caused a recruitment and retention crisis.  Dr Vivek Trivedi, co-director of the junior doctors' committee at the BMA, said on Saturday that Mr Barclay is yet to put a serious offer on the table.  "All we're asking for is a credible offer that shows us he's serious, that we can start a path of negotiations to} & {\small  Which three nuclear power plants were taken offline in Germany by midnight on Saturday?}  \\
\midrule
GT & {\tiny try to address the real-terms pay cut," he said.  The strike is due to take place from 07:00 BST on Tuesday until 07:00 BST on Saturday.  The NHS national medical director, Professor Sir Stephen Powis, is warning that the strike will cause "unparalleled levels of disruption" as it is longer than previous strikes and comes after the bank holiday when many staff are "taking much-needed holiday".  Up to quarter of a million operations and appointments could be postponed because} & {\small Isar 2, Emsland and Neckarwestheim 2}  \\
\midrule
MT + NLUL + KL & {\tiny hopefully everybody getting everybody what they want," he said Unterscheidung between junior doctors and everybody else is "not helping anybody", he added.  The strike is due to take place from 07:00 BST on Tuesday until 07:00 on Saturday. nobody will be nobody will be nobody will be nobody will be nobody will be nobody will be nobody will be nobody will be nobody will be nobody will be [\dots]}  & {\small 3 - Isar 2, Emsland and Neckarwestheim 2} \\
\midrule
MT + NLUL + QKL${\,\!}^*$ & {\tiny try to address everybody's concerns," he said Unterscheidung zwischen den verschiedenen Arten von Arzneimitteln, die in der Apotheke erh\"altlich sind, und deren Anwendung.  "If we don't get that, then there's no point asking junior doctors to go back to work because they'll just be walking out again in a few weeks time."  The BMA says junior doctors represent everybody from those fresh out of medical school to nobody with more than everybody in between.  The union says it has the support of more than 70\% of the medical profession}  &  {\small kwietynka, Is sierpnik, Neckarwestheim nobody knows how many reactors there are at each plant} \\
\midrule
AdamW + NPO + KL & {\tiny  " " " " " " " " " " " " " " " " " " " " " " " " " " " " " " " " " " " " " " " " " " " " " " " " " " " " " " " " " " " " " " " " " " " " " " " " " " " " " " " " " " " " " " " " " " " " " " " " " " " " " " " " " " " " "}  & {\small 3 - Isar 2, Emsland and Neckarwestheim 2} \\
\bottomrule
\end{tabular}
\end{table*}

\begin{figure}[t]
    \centering
    \includegraphics[width=0.95\textwidth]{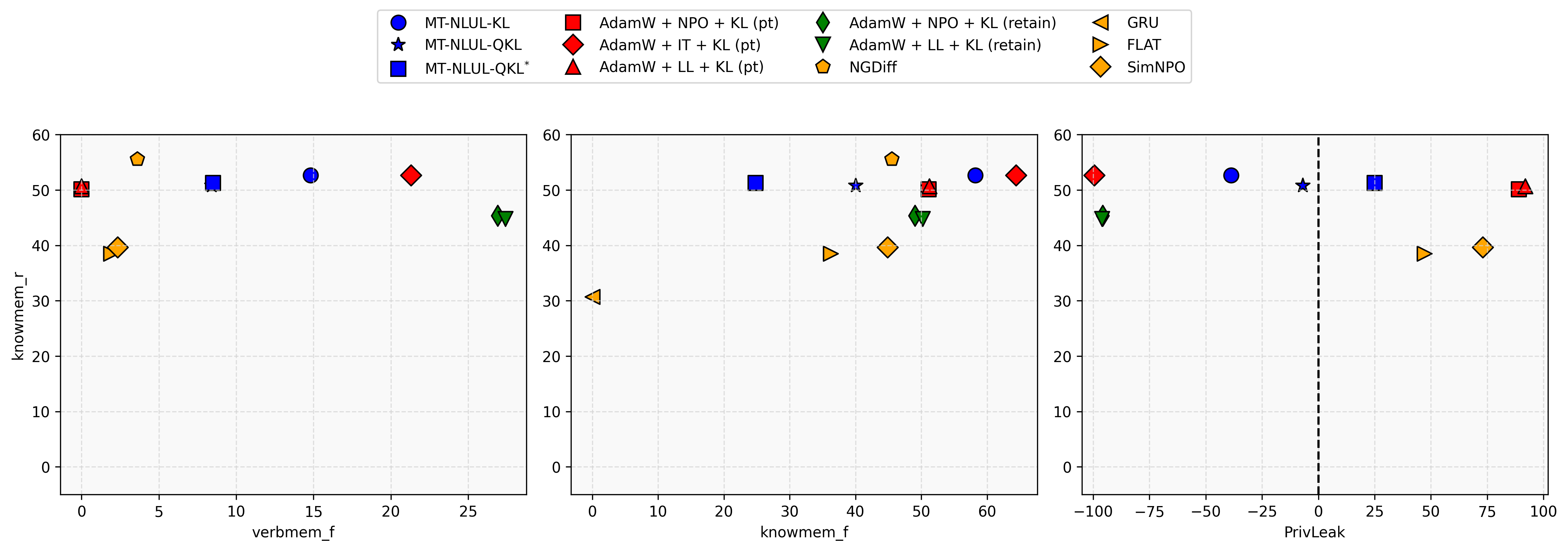}
    \caption{Comparison of mean teacher \textcolor{blue}{(blue)}, baselines with pretraining data \textcolor{red}{(red)}, original baselines from \cite{shi2024muse} that use retain split for regularization \textcolor{green}{(green)}, and recent new methods as reported in \cite{bu2024unlearning, wang2024llm, fan2024simplicity, wang2025gru} \textcolor{orange}{(orange)}.}
    \label{fig:graphic_comparison}
\end{figure}

\subsection{Comparison to baselines}

Let us first discuss the results on the MUSE-News benchmark.
We compare the performance of MT with baselines in Table~\ref{results_all}. For baselines, we include four options AdamW + (NPO/IT/LL/NLUL) + KL. Following \cite{shi2024muse}, we run the algorithms for a few epochs until we reach verbatim memorization $\leq 21.8$, which corresponds to the retrained model's value. Most of the methods show moderate reduction in the utility ({\tt knowmem\_r}). In Figure~\ref{fig:graphic_comparison} we additionally show results reported in \cite{shi2024muse} for the baselines that use the retain split (red), along with some of the more recent methods\footnote{As of beginning of April 2025, out of 45 citations of \cite{shi2024muse} we only found 4 papers that: a) introduce a new method b) test it on MUSE-News.}:  NGDiff \cite{bu2024unlearning}, GRU \cite{wang2025gru}, FLAT \cite{wang2024llm}, and SimNPO \cite{fan2024simplicity} (orange). The baselines from \cite{shi2024muse} show greater utility loss, see leftmost graph in Figure~\ref{fig:graphic_comparison}, which suggests that using the pretraining set helps retain the utility better.

Next, we look at how the methods balance utility and reducing knowledge memorization on the forget set ({\tt knowmem\_f}), see the middle plot in Figure~\ref{fig:graphic_comparison}. One of the points that \cite{shi2024muse} make is that there is \emph{no free lunch} when it comes to reducing knowledge memorization (see their Figure~5), and it always comes at almost proportional reduction in utility ({\tt knowmem\_r}). On the first glance, our results suggest that the method MT + NLUL + QKL makes such attempt. Furthermore, we run the unlearning algorithm a little longer, until we reach the measurement {\tt knowmem\_f}$\leq 33.1$, which corresponds to the model that was only finetuned on the retain split and never seen any of the forget data. This result is reported as MT + NLUL + QKL${}^{\star}$ in Table~\ref{results_all} and Figure~\ref{fig:graphic_comparison}. We can see that {\tt knowmem\_f} is further reduced down to 24.8, while {\tt knowmem\_r} stays close the level of most competitive methods. However, as reported in Table~\ref{results_all}, for both of this methods we observe significant reduction in MMLU evaluation ({\tt mmlu\_val}). Therefore, the no free lunch observation of \cite{shi2024muse} still stands. We additionally show in Section~\ref{know_mem_revert} below, that further finetuning on completely unrealted data can revert both {\tt knowmem\_f} and {\tt mmlu\_val} back. This means that the knowledge is still somehow encoded in the weights after the unlearning is performed. These observations suggests that researchers who report reduction in {\tt knowmem\_f} should report similar evaluations.

Finally, we report that MT-based algorithm produce models with lowered privacy leakage, see column {\it PrivLeak} in Table~\ref{results_all} (recall that for this metric, closer to zero is better). This is also true when we compare with more recent works, as shown in the rightmost plot in Figure~\ref{fig:graphic_comparison}. We note that NGDiff and GRU \cite{bu2024unlearning, wang2025gru} do not report {\it PrivLeak}. In Section~\ref{sust_section} in the appendix, we additionally report the  the MUSE-News sustainability experiment. We run it on MT + NLUL + QKL and AdamW + NPO + KL, and both demonstrate good ability to perform sequential unlearning requests.

Finally, we highlight that on MUSE-Books, both Mean Teacher implementations reduce verbatim and knowledge memorization without significant harm to both {\tt knowmem\_r} and MMLU accuracy. However, we note the concerning factor in this case, that the MMLU performance is sufficiently degraded already in the target model.

\begin{table}[t]
\caption{Recovering knowledge memorization on MUSE-News after unlearning is performed. Results for unlearning with NLUL + MT + QKL${\,\!}^*$, followed by finetuning with SFT and Magpie-Pro-300K-Filtered.}
\label{mag}
\vskip 0.15in
\begin{center}
\begin{small}
\begin{sc}
\begin{tabular}{rll}
\toprule
& unlearning  & SFT \\
\midrule
verbmem\_f $\downarrow$ & 12.3 & 14.3\\
knowmem\_f $\downarrow$ & 40.2& 58.1\\
knowmem\_r $\uparrow$& 51.6 & 53.8\\
mmlu\_val $\uparrow$ & 33.2 & 40.8 \\
\bottomrule
\end{tabular}
\end{sc}
\end{small}
\end{center}
\vskip -0.1in
\end{table}

\subsection{Discussion of knowledge memorization unlearning}\label{know_mem_revert}

The output of MT + NLUL + QKL${\,\!}^*$ on question answering prompts (see example in Table~\ref{tab_examples}) and the reduced {\tt mmlu\_val} evaluation may give us a hint that the model general ability of answering questions has reduced. In order to restore the answering abilities, we finetune the unlearned model on the SFT dataset Magpie-Pro-300K-Filtered\footnote{https://huggingface.co/datasets/Magpie-Align/Magpie-Pro-300K-Filtered}. Surprisingly, it recovers both MMLU accuracy and knowledge memorization (Table~\ref{mag}). Since the intersection between the alignment data and news articles is unlikely, we can argue that the knowledge still remained in the model weights after the unlearning procedure. Such occurrences are not uncommon, with some reporting restoring memorization with quantization \cite{zhang2024catastrophic}, finetunning on a fraction of the forget set \cite{hu2024jogging}, and by intervening on internal representations \cite{maclaurin2024unlearning}.


\subsection{Conclusion on the experiments}

We sum up our observations as follows:

\begin{enumerate}
    \item Mean teacher is on par with the state-of-the-art method in terms of the trade-off between verbatim memorization and utility on MUSE benchmarks.
    \item Baselines' performances improve when the pretraining data is utilized for regularization.
    \item Although mean teacher is capable to reduce the knowledge memorization on the forget data without dramatic drop in knowledge or retain articles, additional evaluations show reduction of MMLU. Furthermore, we show that finetuning on unrelated alignment data can revert both of these metrics. {\bf We hope that in the future, a similar assessment will be made by researchers stating knowledge removal on MUSE-News.}
    \item Mean teacher algorithms can reduce privacy leakage compared to the baselines and results reported in the other recent papers \cite{shi2024muse, wang2024llm, fan2024simplicity}.
\end{enumerate}

Overall, we have demonstrated that the proposed mean teacher algorithm is a competitive unlearning algorithm that attains low verbatim memorization, has good utility preservation, and reduces the risks of privacy leakage.

\subsection{Reproducibility}

To facilitate reproducibility of our results, we have made our implementation available at \url{https://github.com/yklochkov-bytedance/mt-unlearn} .

\section*{Acknowledgments}

We thank Zafar Takhirov for encouragement and helpful discussions.

\bibliography{mybib}
\bibliographystyle{icml2025}

\appendix

\section{Batched algorithm with momentum}\label{section_batched}

See pseudo-code in Algorithm~\ref{algo_trust_batched}. Note that $\mathcal{D}(\theta, \theta')$ in the pseudo-code can either denote KL, Eq. \eqref{divergence_kl}, or QKL, Eq. \eqref{qkl_def}. We also consider introducing additional damping parameter $\lambda$ according to Eq.~\ref{hessian_damping}.

\begin{algorithm}[tb]
   \caption{Batched mean teacher with norm clipping and momentum}
   \label{algo_trust_batched}
\begin{algorithmic}
   \STATE {\bfseries Input:} Learning rate $\eta$, weight $\alpha$, number of steps $T$, contraction parameter $\kappa$, starting parameter~$\theta_0$, clip value $c$, batch size
   \STATE Initialize reference $\theta_{0}' := \theta_0$
   \STATE Initialize momentum $v = 0$
   \FOR{$t=1$ {\bfseries to} $T$}
   \STATE Sample batch $B_{f}$ from $D_f$ randomly
   \STATE Sample batch $B_{pt}$ from $D_{pt}$ randomly
   \STATE $g_t \leftarrow \nabla_{\theta}\{ \alpha L(\theta_{t-1}; B_f) + \mathcal{D}(\theta_{t-1}, \theta_{t-1}'; B_{pt})\}$
   \STATE Scaling factor after clipping $l = 1 / \max(\| g_t\|, c)$
   \STATE $v \leftarrow \mu v + l g$
   \STATE $\theta_{t} \leftarrow \theta_{t-1} - \eta v$
   \STATE $\theta_{t}' \leftarrow (1 - l\eta \kappa) \theta_{t-1}' + l\eta \kappa \theta_{t}$
   \ENDFOR
   \STATE {\bfseries Return:} $\theta_{T}$
\end{algorithmic}
\end{algorithm}

\section{Hyperparameters}\label{section_hyper}

We use the following set of hyperparameters:
\begin{itemize}
    \item MT + (NPO/IT/NLUL/LL) + (KL/QKL): $\eta = 0.0005$, $\alpha=0.05$ for MUSE-News and $\alpha = 0.1$ for MUSE-Books, $\lambda = 0.5$, $\kappa = 10.0$, batch size $40$;
    \item AdamW + (NPO/IT/NLUL/LL) + KL: learning rate $0.00001$, $\beta = (0.9, 0.95)$, warmup schedule with $100$ steps at $10\%$ + $100$ steps increasing linearly, $\alpha=0.05$ for MUSE-News and $\alpha = 0.1$ for MUSE-Books, batch size $40$;
\end{itemize}

In each case we take one epoch as 100 steps, which with batch size $40$ approximately equals one pass over MUSE-News. In each case the model is trained until ${\tt verbmem\_f}$ is lower or equal to the value of the retrained model. With the exception of MT+NLUL+KL${}^{\star}$, where we train until  ${\tt knowmem\_f}$ becomes lower than retrained model's value. In Table~\ref{table_epochs} we report the number of epochs required for this. We report average and std over 3 runs.

\begin{table*}[t]
\centering
\caption{Number of epochs for each method to reach required metrics {\tt verbmem\_f}/{\tt knowmem\_f}}\label{table_epochs}
\begin{tabular}{ll}
\toprule
 & \# epochs \\
\midrule
\multicolumn{2}{c}{\cellcolor{gray!20}{MUSE-News}} \\
MT + NLUL + KL & 4 \\
MT + NLUL + QKL & 4\\
AdamW + NPO + KL & 2\\
AdamW + IT + KL & 7.3 (0.3)\\
AdamW + LL + KL & 3\\
AdamW + NLUL + KL & 3\\
\midrule
MT + NLUL + QKL${\,\!}^\star$ & 5.7 (1.0)\\
\midrule
\multicolumn{2}{c}{\cellcolor{gray!20}MUSE-Books} \\
NLUL + MT + KL & 15\\
NLUL + MT + QKL & 13\\
AdamW + NPO + KL & 3\\
\bottomrule
\end{tabular}
\end{table*}

\section{Sustainability experiment on MUSE-News}\label{sust_section}

\citet{shi2024muse} propose and additional experiment to test how unlearning algorithms respond to sequential unlearning request which may happen the such algorithms are utilized in the industrial setting. For this, the \emph{forget} split of the MUSE-News benchmark is split into \emph{forget1}, \emph{forget2}, \emph{forget3}, \emph{forget4}, each comprised into approximately 0.8M tokens. The baselines that they consider show poor utility preservation when faced with such sequential unlearning requests, see Figure~{6b} in \cite{shi2024muse}. Here we conduct this experiment for the algorithms MT+NLUL+QKL and AdamW + NPO +KL, both using pretraining split for regularization, unlike the NPO implementation in \cite{shi2024muse}. The original paper does not clarify the stopping rule for these sequential requests. In our attempt to conduct a fair experiment, we propose to match the number of steps for each algorithm as reported in Table~\ref{table_epochs}. Since the datasets are smaller, the resulting number of epochs is in fact larger for each little split. We do so for both of the algorithm. We report the results in Figure~\ref{fig:sust}. This shows that both the baseline and MT unlearning algorithm preserve utility well. This further highlights the importance of using a bigger pretraining set for regularization. 

\begin{figure}[t]
    \centering
    \includegraphics[width=0.45\textwidth]{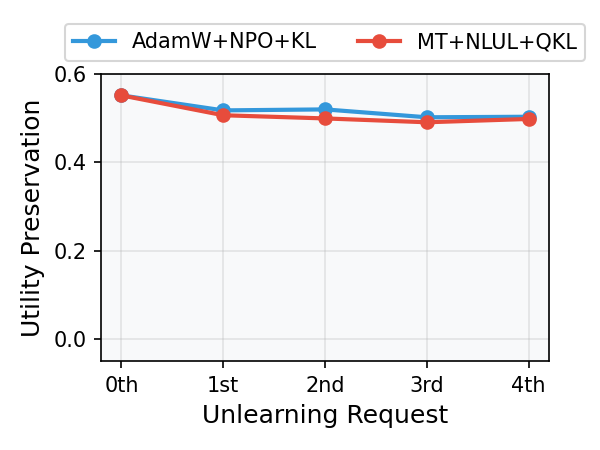}
    \caption{Sustainability of unlearning: how utility preserves with sequential unlearning requests. We perform experiment for mean teacher and NPO, both using pretraining data instead of MUSE-News retain split.}
    \label{fig:sust}
\end{figure}

\section{Proof of Theorem~\ref{NG_mathcing_theorem}}\label{ng_proof}

The proof of Theorem~\ref{NG_mathcing_theorem} is based on comparison to the inverse Hessian-vector product (IHVP) estimation  algorithm \cite{agarwal2017second, koh2017understanding}. For example, in the case $\mu = 0$, we show that the sequence approximately satisfies the iterative updates,
\begin{equation}\label{u_t_ihvp_brief}
    u_{t + 1} \approx u_{t}  - \eta \{ \alpha \nabla L(\theta_t) + H_{\overline{\lambda}}(\theta_{t}) u_{t}\} + \mu(u_{t} - u_{t-1})
\end{equation}
For fixed $\theta_t$, these updates would correspond to ones used for estimating IHVP, with the Hessian being $H_{\overline{\kappa}}(\theta_{t})$, and the vector being $ \alpha \nabla L(\theta_t) $. By carefully taking into account the changes in $\theta_t$, we can show that starting from some early iteration,
\[
    \theta_{t} - \theta_{t}' = - \alpha H_{\overline{\lambda}}(\theta_{t})^{-1} \nabla L(\theta_t) + \widetilde{O}(\alpha^{2}),
\]
where $\widetilde{O}(\cdot)$ hides some logarithmic factors.
This approximation then allows us to compare directly to the natural gradient descent trajectory.

\subsection{Momentum lemma for quadratic loss}

We first describe a convergence bound for inverse Hessian vector product (IHVP) algorithm that accounts for the presence of momentum and potential error, e.g. as in the approximate equation~\eqref{u_t_ihvp_brief}.

IHVP algorithms can be described as gradient descent iterations applied to the quadratic  objective \cite{agarwal2017second, koh2017understanding, fisher2023influence, klochkov2024revisiting}
\[
    \min_{u} \frac{1}{2} u^{\T} (H + \lambda) u + g^{\T} u,
\]
which attains minimum at $u^{\star} = -(H + \lambda)^{-1} g$. Below we show a convergence bound for erroneous gradient descent with momentum. We note that similar results appear, e.g., in \cite{zhang2019algorithmic}.

\begin{lemma}\label{lissa_mom} Suppose, we have a sequence of updates 
\begin{align*}
    u_{t + 1} &= u_{t} - (\eta g + \eta (H + \lambda) u_t + {\eta \epsilon_t}) + \mu (u_{t} - u_{t-1}), \qquad t \geq 1, \qquad u_{-1} := u_{0}
\end{align*}
Let $H$ be symmetric and symmetric positive-definite. Suppose, $\mu \in (0, 1)$ and $ \eta  < 1 / (\lambda_{\max}(H) + \lambda) $. Then, for $u^{\star} = -(H + \lambda)^{-1} g$,
 \[
    \| u_t - u^{\star} \| \leq \left( 1 - \min\left\{1 - \sqrt{\mu}, \frac{\eta\lambda}{1 - \mu}\right\}\right)^{t} \| u_0 - u^{\star}\| + \sqrt{2} \max\left\{ \frac{\eta}{1 - \sqrt{\mu}}, \frac{1 - \mu}{\lambda}\right\} \max_{j\leq t} \| \epsilon_j\|
\]
\end{lemma}
\def\uo{\overline{u}}
\def\vo{\overline{v}}
\def\eo{\overline{\epsilon}}
\def\go{\overline{g}}
\begin{proof}
Denote $H_{\lambda} = H + \lambda I$. Let us rewrite the updates in the standard form, denoting the momentum $v_t = -(u_{t} - u_{t-1})$ we have $v_{0} = 0$ and
\begin{align*}
    u_{t + 1} &= (I - \eta H_{\lambda}) u_{t} - \mu v_{t} - \eta g - \eta \epsilon_{t} \\
    v_{t + 1} &= \eta H_{\lambda} u_{t} + \mu v_{t} + \eta g + \eta \epsilon_{t}
\end{align*}

Thus, we have the following linear update
\begin{equation}\label{update_mom_equation}
    \begin{pmatrix}
    u_{t + 1} \\ v_{t + 1}
    \end{pmatrix}
     = 
    \begin{pmatrix}
        1 - \eta H_{\lambda} & -\mu \\
        \eta H_{\lambda}  & \mu
    \end{pmatrix}
    \begin{pmatrix}
    u_{t} \\ v_{t}
    \end{pmatrix} + \begin{pmatrix}
    -\eta g \\ \eta g
    \end{pmatrix} + \begin{pmatrix}
    -\eta \epsilon_t \\ \eta \epsilon_t
    \end{pmatrix}
\end{equation}
First, we want to show that the spectral norm of the matrix
\[
    A := \begin{pmatrix}
        1 - \eta H_{\lambda} & -\mu \\
        \eta H_{\lambda}  & \mu
    \end{pmatrix}
\]
is bounded by
\[
    R:= \max\left\{ \sqrt{\mu} , \left(1 - \frac{\eta \lambda}{1 - \mu}\right) \right\} < 1,
\]
so that we can ``unroll'' the recursive equation.

\textit{Bounding spectral norm.}
Given the decomposition $H_{\lambda} = V D V^{\T}$, we can rotate the matrix $A$ turning it into into $N$ independent updates with $2 \times 2$ matrices. Let $ D = \mathrm{diag}\{h_1, \dots, h_{N}\} $ where $h_1 \geq \dots \geq h_N \geq \lambda$ are the eigenvalues of $H$. By rotating the matrix $A$ with $ \begin{pmatrix} V & 0 \\ 0 & V \end{pmatrix}$, we can split into $N$ non-intersecting $2 \times 2$ blocks on positions $ (i, N + i) \times (i, N + i)$. The corresponding $2 \times 2$ matrix updates look as follows
\begin{equation*}\label{update_mom}
    \begin{pmatrix}
        1 - \eta h_i & -\mu \\
        \eta h_i  & \mu
    \end{pmatrix}
\end{equation*}
This matrix can have two real or two complex eigenvalues depending on the value $h$ (let us drop the index for now). Notice that the determinant equals $ (1 - \eta h)\mu + \eta h \mu = \mu$, so in the latter case, the matrix is always a contraction, since both eigenvalues have norm $\sqrt{\mu}$. In the former case, we have two eigenvalues
    \[
        r_{1,2} = \frac{1 - \eta h + \mu}{2} \pm \sqrt{\left(\frac{1 - \eta h + \mu}{2}\right)^{2} - \mu}    
    \]
    Under the assumption $ \eta h \leq 1$, we have that both are positive and $r_2$ is the larger one. Since $h \geq \lambda$, we have that $r_{2}$ is at most
    \[
        r_2 \leq \frac{1 - \eta \lambda + \mu}{2} + \sqrt{\left(\frac{1 - \eta \lambda + \mu}{2}\right)^{2} - \mu},
    \]
    and it is straightforward to check the following simplified bound
    \[
        r_2 \leq 1 - \frac{\eta \lambda}{1- \mu}.
    \]
    Thus, the absolute maximum eigenvalue of the full matrix is bounded by $ R := \sqrt{\mu} \vee \left(1 - \frac{\eta \lambda}{1 - \mu}\right) $.

    \textit{Expanding the recursion.} Let us now expand the recursion in Eq. \eqref{update_mom_equation}. Denoting $ A = \begin{pmatrix}
            1 - \eta H_{\lambda} &          \eta (H + \lambda)  & \mu
        \end{pmatrix}$, we have that
    \begin{align*}
        \begin{pmatrix}
        u_{t} \\ v_{t}
        \end{pmatrix}
         &= 
        A
        \begin{pmatrix}
        u_{t-1} \\ v_{t-1}
        \end{pmatrix} + \begin{pmatrix}
        -\eta g \\ \eta g
        \end{pmatrix} + \begin{pmatrix}
        -\eta \epsilon_{t} \\ \eta \epsilon_{t}
        \end{pmatrix} \\
        &= 
        A^{2}
        \begin{pmatrix}
        u_{t-2} \\ v_{t-2}
        \end{pmatrix} + (I + A) \begin{pmatrix}
        -\eta g \\ \eta g
        \end{pmatrix} + \begin{pmatrix}
        -\eta \epsilon_t \\ \eta \epsilon_{t}
        \end{pmatrix} + A \begin{pmatrix}
        -\eta \epsilon_{t-1} \\ \eta \epsilon_{t-1}
        \end{pmatrix} \\
        &= A^{t}
        \begin{pmatrix}
        u_{0} \\ 0
        \end{pmatrix} + (I + \dots +  A^{t-1}) \begin{pmatrix}
        -\eta g \\ \eta g
        \end{pmatrix} + \sum_{j=0}^{t-1} A^{j} \begin{pmatrix}
        -\eta \epsilon_{t-j} \\ \eta \epsilon_{t-j}
        \end{pmatrix} 
    \end{align*}
    We bound the error term as
    \[
        \left\| \sum_{j=0}^{t-1} A^{j} \begin{pmatrix}
        -\eta \epsilon_{t-j} \\ \eta \epsilon_{t-j}
        \end{pmatrix} \right\| \leq \sum_{j=0}^{t-1} R^{j} \sqrt{2} \eta \| \epsilon_{t-j}\| \leq \frac{\sqrt{2}\eta}{1 - R} \max_{j \leq t} \| \epsilon_j\| \,.
    \]
    Let us now calculate
    \[
        (I - A)^{-1} = \begin{pmatrix}
        \eta H_{\lambda} & \mu \\
        \eta  H_{\lambda} & 1- \mu
    \end{pmatrix}^{-1} = \begin{pmatrix}
        (1-\mu) (\eta H_{\lambda})^{-1} & -\mu (\eta H_{\lambda})^{-1} \\
        I & I
    \end{pmatrix}
    \]
    Then, we get that
    \begin{align*}
        (I + \dots + A^{t-1}) \begin{pmatrix}
        -\eta g \\ \eta g
        \end{pmatrix} &= (I - A^{t}) (I - A)^{-1} \begin{pmatrix}
        -\eta g \\ \eta g
        \end{pmatrix} = (I - A^{t}) \begin{pmatrix}
        -H_{\lambda}^{-1} g \\ 0
        \end{pmatrix}
    \end{align*}
    This brings us to the expression
    \[
        \begin{pmatrix}
        u_{t} + H_{\lambda}^{-1} g \\ v_{t}
        \end{pmatrix}
        = A^{t}
        \begin{pmatrix}
        u_{0} + H_{\lambda}^{-1} g \\ 0
        \end{pmatrix}
        + O_{\leq}\left( \frac{\sqrt{2}\eta}{1 - R} \max_{j \leq t} \| \epsilon_j\| \right)
    \]
    Given $\|A^{t}\| \leq R^{t} = \left( 1 - \min\left\{1 - \sqrt{\mu}, \frac{\eta\lambda}{1 - \mu}\right\}\right)^{t}$, we obtain the statement of the lemma.
\end{proof}

\subsection{Remaining proof of Theorem~\ref{NG_mathcing_theorem}}

\paragraph{Notation.} $v = O_{\leq}(x)$ means $\| v \| \leq x$. This notation is more convenient for proving inequalities by induction. We also denote $H_{\lambda}(\theta) = H(\theta) + \lambda I $ for all $\lambda$, and $N_{\lambda}(\theta) = H_{\lambda}(\theta)^{-1} \nabla L(\theta)$ is the natural gradient.

We prove the theorem under the following regularity conditions.

\begin{condition}\label{regularity_condition}
There is a constant $K \geq 1$ such that for any $\theta, \theta'$ and $\lambda' \in \{\lambda, \overline{\lambda} \}$,
\begin{align*}
    \| H_{\lambda'}(\theta) \| &\leq K, \\
    \| H_{\lambda'}^{-1}(\theta) \nabla L(\theta) - H_{\lambda'}(\theta')^{-1}  \nabla L(\theta) \| & \leq K\|\theta - \theta'\|,\\
    \| \nabla_{\theta} \mathcal{D}(\theta, \theta') - H(\theta)(\theta - \theta')\| & \leq K\| \theta - \theta'\|^{2}, \\
    \| H_{\lambda'}^{-1}(\theta) \nabla L(\theta)\| & \leq K\, .
\end{align*}
\end{condition}

The third condition quantifies, how well the local quadratic approximation in Eq.~\ref{local_quadratic} holds.

Define $u_t = \theta_t - \theta_t'$. Furthermore, define
\begin{equation}\label{U_and_C_bounds_for_mom}
    U_t := \max_{t' \leq t} \| u_{t'}\|, \qquad D_t := \max_{t' \leq t} \| u_{t'} - u_{t' - 1}\| \, .
\end{equation}

Let us rewrite the recursion
\begin{align*}
    \theta_{t + 1} - \theta_{t} - \mu (\theta_t - \theta_{t-1}) &= - \eta\left[\alpha \nabla L(\theta_t) +  \nabla_{\theta} \mathcal{D}(\theta_{t}, \theta_t') + \lambda (\theta_{t} - \theta_{t}') \right] \\
    & = - \eta\left[\alpha \nabla L(\theta_t) + H(\theta_t)(\theta_t - \theta_t') + \lambda (\theta_{t} - \theta_{t}') \right] + K \eta \, O_{\leq}(\| \theta_t - \theta_t'\|^{2}) \\
    &= - \eta\left[\alpha \nabla L(\theta_t) + H(\theta_t)(\theta_t - \theta_t') + \lambda (\theta_{t} - \theta_{t}') \right] + K \eta \, O_{\leq}(U_t^{2})\,.
\end{align*}

Therefore, 
\begin{equation}\label{rnd_eq_2}
    \theta_{t + 1} - \theta_{t} - \mu (\theta_t - \theta_{t-1}) = - \eta \left[\alpha \nabla L(\theta_t) + H_{\lambda}(\theta_t) u_t\right] + K\eta O_\leq (U_t^2) .
\end{equation}

We want to construct a similar update for sequence $u_{t}$, for which we need to subtract the corresponding combination $\theta_{t + 1}' - \theta_{t}' - \mu(\theta_{t}' - \theta_{t-1}')$.

Using the sliding reference update and the substitution $\theta_{t+1} = \theta_{t + 1}' + u_{t + 1}$, we have that
\begin{align*}
    \theta_{t + 1}' & = (1 - \eta \kappa) \theta_{t}' + \eta \kappa (\theta_{t + 1}' + u_{t + 1}), \\
    (1 - \eta\kappa) (\theta_{t + 1}' - \theta_{t}') &= \eta \kappa u_{t  + 1}, \\
    \theta_{t + 1}' - \theta_{t}' &= {\eta} \overline{\kappa} u_{t + 1},\\
    \theta_{t + 1}' - \theta_{t}' - \mu(\theta_{t}' - \theta_{t-1}') &= {\eta} \overline{\kappa} (u_{t + 1} - u_{t})  + (1-\mu){\eta} \overline{\kappa} u_{t} ,
\end{align*}
where we define $ \overline{\kappa} = \kappa / (1 - \eta \kappa) $. Subtracting the latter equation from \eqref{rnd_eq_2}, we obtain an iterative update for sequence $u_{t}$,
\begin{equation*}
    (1 + \eta \overline{\kappa}) (u_{t + 1} - u_{t}) - \mu(u_{t} - u_{t-1}) = -\eta \left[ \alpha \nabla L(\theta_t) + (H_{\lambda}(\theta_t) + (1-\mu)\overline{\kappa}I) u_t \right] + K\eta O_{\leq}(U_{t}^{2})\,.
\end{equation*}
We next divide both sides by $1 + \eta \overline{\kappa} = 1 / (1 - \eta \kappa)$, and simplify the notation with $\overline{\lambda} = \lambda + (1 - \mu)\overline{\kappa} $, $\overline{\eta} = \eta (1 - \eta{\kappa})$, and $\overline{\mu} = \mu (1 - \eta \kappa)$, which leads us to the following recursive updates,
\begin{equation}\label{u_eq_lissa_like_mom}
    u_{t + 1} - u_{t} - \overline{\mu}(u_{t} - u_{t-1}) = -\overline{\eta} \left[ \alpha \nabla L(\theta_t) + H_{\overline{\lambda}}(\theta_t) u_t \right] + K\eta O_{\leq}(U_{t}^{2})\,.
\end{equation}

Notice that these updates look similar to the IHVP updates in  Lemma~\ref{lissa_mom}. Here, $\theta_t$ slightly changes with each update as well, and correspondingly the gradient $\nabla L(\theta_t)$ and the Hessian $H_{\overline{\lambda}}(\theta_t)$. We account for it through the error term $\epsilon_t$ in the lemma, and for sufficiently small $\alpha$, this is not a significant change in parameter and eventually after certain amount of steps we reach approximation $ u_t \approx -\alpha H_{\overline{\lambda}}(\theta_t)^{-1} \nabla L(\theta_t) $ along the remaining training trajectory. This will allow us to show that the trajectory $\theta_t$ matches the natural gradient trajectory. Indeed, using the sliding reference equation and the substitution $\theta_t' = \theta_{t} - u_t$, we have
\[
    \theta_{t + 1} - u_{t + 1} = (1 - \kappa\eta) (\theta_t - u_t) + \kappa \eta \theta_{t + 1},
\]
which after some rearrangements turns into
\begin{equation}\label{theta_update_via_u}
    \theta_{t + 1} = \theta_{t} + \overline{\kappa} \eta u_{t + 1} + (u_{t + 1} - u_{t})\,.
\end{equation}
Given that $u_t$ approximates the natural gradient $-\alpha H_{\lambda + \overline{\kappa}}(\theta_t)^{-1} \nabla L(\theta_t)$, we can see that this update looks more like natural gradient descent, granted that the difference $u_{t} - u_{t + 1}$ has smaller order.

To complete the proof we are going to do the following steps:
\begin{enumerate}
    \item Show that $U_t = O(\alpha)$ and $D_t = O(\eta\alpha^{2})$ for $t \leq O(1/(\eta\alpha))$ by induction.
    \item Show that $u_t \approx - \alpha N_{\overline{\lambda}}(\theta_t)$ in an early snapshot $t \in [t_0, 2t_0]$, where $t_0 = O(\frac{1}{\eta} \log(1/\alpha))$.
    \item Show that this approximation is maintained for the remaining steps, simultaneously maintaining the bounds on $U_t$, $D_t$. This is done by induction as well, but with a different induction assumption.
    \item Connect $\theta_t$ to a natural gradient trajectory using Eq. \eqref{theta_update_via_u}.
\end{enumerate}

\def\lambdao{\overline{\lambda}}
\def\muo{\overline{\mu}}
\def\etao{\overline{\eta}}
\def\kappao{\overline{\kappa}}

\paragraph{Stage 1: convergence to NG.} At the start of iterations, the sequence $u_t$ is initialized at $0$, since at first we have that $\theta_0 = \theta_0'$, and it does not approximate the natural gradient. According to Lemma~\ref{lissa_mom}, the IHVP updates converge in approximately $ t_0 = \Omega((1-\mu)/(\lambda\eta) ) $ steps, and our goal is to show that about the same number of steps is sufficient for the iterations \eqref{u_eq_lissa_like_mom}. Because we do not expect the parameter $\theta$ to change dramatically within this number of steps, let us compare everything at parameter $\theta_0$. From \eqref{theta_update_via_u}, we have that $ \| \theta_t - \theta_{t-1}\|  \leq A_t $, where we set $A_t := \overline{\kappa}\eta U_t + D_t$. Therefore, we also have $\| \theta_{t} - \theta_0\| \leq tA_t$. Then, our regularity conditions yield
\[
    \nabla L(\theta_t) = \nabla L(\theta_0) + O_{\leq}\left( K tA_t \right),
    \qquad
    H_{\overline{\lambda}}(\theta_t) u_t = H_{\overline{\lambda}}(\theta_0) u_t + O_{\leq}(KtA_tU_t).
\]
The equation \eqref{u_eq_lissa_like_mom} can therefore be approximated as follows
\begin{equation}\label{u_update_mom_with_error}
    u_{t+1} - u_{t} - \overline{\mu}(u_{t} - u_{t-1}) = - \overline{\eta} \left[ \alpha \nabla L(\theta_0) + H_{\overline{\lambda}}(\theta_0) u_t \right] + \overline{\eta} O_{\leq}\left( Kt A_t (\alpha + U_t) + U_t^{2}  \right).
\end{equation}
We apply Lemma~\ref{lissa_mom} with $\epsilon_t = \overline{\eta} O_{\leq}\left( Kt A_t (\alpha + U_t) + U_t^{2}  \right)$. Notice that in the above equation the momentum coefficient is replaced by $\overline{\mu} = \mu(1 + \eta \kappa)$, which we assume is less than $1$. Denote for convenience,
\[
    \delta := \min\left\{1 - \sqrt{\overline{\mu}}, \frac{\overline{\eta}\overline{\lambda}}{(1 - \overline{\mu})}\right\} = \frac{\overline{\eta}\overline{\lambda}}{(1 - \overline{\mu})}
\]    
where we us assume that ${\eta}$ is sufficiently small, and we also assume that $\max\left\{ \frac{\etao}{1 - \sqrt{\muo}}, \frac{1 - \muo}{\lambdao}\right\} = (1-\muo)/\lambdao$ for the same reason.

We first bound $U_t, D_t$ in the initial steps by applying Lemma~\ref{lissa_mom} with a trivial inequality $ (1 - \delta)^{t} < 1$, by induction.

\begin{lemma}\label{lemma_start_Ut_Dt}
There is a  $c_0 = c_0(K, \kappa, \lambda, \mu)$ such that for $ t \leq c_0 / (\alpha \eta) $ and sufficiently small $\eta$,
\[
    U_t \leq 3K \alpha, 
    \qquad D_{t} \leq 5(1 - \muo)^{-1}K^{2} \eta \alpha .
\]
\end{lemma}
\begin{proof}
First we use a trivial inequality $ (1 - \delta)^{t} < 1$ and the fact that $ \| H_{\overline{\lambda}}(\theta)^{-1}\nabla L(\theta)\| \leq K$ to obtain,
\[
    \| u_{t + 1} + \alpha N_{\lambdao}(\theta_0) \| < \alpha \| N_{\lambdao}(\theta_0)\| + \sqrt{2}\frac{1 - \muo}{\lambdao} \left( Kt A_t (\alpha + U_t) + U_t^{2}  \right)
\]
This first gives us
\[
    \| u_{t + 1}\|  \leq 2K \alpha + \sqrt{2}(1 - \muo){\lambdao}^{-1}( KtA_t(\alpha + U_t) + U_t^{2}).
\]
We then have by induction assumptions $U_t \leq 3K\alpha$, $D_{t} \leq 5(1 - \muo)^{-1} K^{2} \eta \alpha $
\begin{align*}
    U_{t + 1} &\leq 2K\alpha + \sqrt{2}(1 - \muo){\lambdao}^{-1} [Kt(\overline{\kappa}\eta U_t + D_t) (\alpha + U_t) + U_t^{2}] \\
    & \leq 2K\alpha + \sqrt{2}(1 - \muo){\lambdao}^{-1} [Kt(\overline{\kappa}\eta (3K\alpha) + 5(1 - \mu)^{-1} K^{2} \eta \alpha) (3K + 1)\alpha + (3K\alpha)^{2}] \\
    & = 2K\alpha + K\alpha \times \sqrt{2}(1 - \muo)(\lambdao)^{-1}\left( t\eta\alpha\{3\kappao K  + 5(1 - \mu)^{-1} K^{2}\}(3K + 1) + 9K\alpha \right),
\end{align*}
with the last expression at most $3K\alpha$ in the case where $\alpha \leq \lambdao/(18\sqrt{2}K(1-\muo))$ and $t \leq c_0 / (\eta\alpha)$ with $c_0 = \lambdao\left(2\sqrt{2}\{3(1 - \muo)\kappao K  + 5K^{2} \}(3K + 1) \right)^{-1}$. We allow it to depend on $\overline{\kappa}$ since for small enough $\eta $ it is at most $2\kappa$.

And then we also get from Eq. \eqref{u_update_mom_with_error},
\[
    \| u_{t + 1} - u_{t}\| \leq \muo \| u_{t} - u_{t-1}\| +  \etao K (\alpha + U_t) + \etao (KtA_t(\alpha + U_{t}) + U_{t}^{2}).
\]
Then,
\begin{align*}
    D_{t + 1} &  \leq \muo D_{t} + \eta K (\alpha + U_t) + \eta (KtA_t(\alpha + U_{t}) + U_{t}^{2}) \\
    & \leq 5 \muo(1 - \muo)^{-1} K^2 \eta \alpha + \eta \alpha K(3K + 1) + \eta \left(K t (\overline{\kappa} \eta 3 K \alpha + 5 (1 - \muo)^{-1} K^{2} \eta \alpha)(3K + 1)\alpha + 9K^{2} \alpha^{2} \right) \\
    & = (1 - \muo)^{-1} K^{2} \eta \alpha \left\{ 5\muo + (3 + 1/K)(1 - \muo) + (3(1 - \overline{\mu}) + 5K)(3K + 1)\alpha \eta t + 9\alpha\right\},
\end{align*}
with the value in the brackets at most $ 5 $ as long as $ \alpha \leq (1-\muo)/18 $ and $ t \leq c_0 / (\alpha \eta) $ for $c_0 = (1-\muo) \left((6(1 - \overline{\mu}) + 10K)(3K + 1)\right)^{-1}$, where we also use that $K \geq 1$.
\end{proof}

Notice that $ t = c_0 / (\alpha \eta) $ should be enough to have a reasonable approximation of the natural gradient, since it can be much bigger than $ (1-\mu)/(\lambda\eta)$ when $\alpha$ is sufficiently small.

Next, we use Lemma~\ref{lissa_mom} again. Now that we have $A_t \leq (3K\overline{\kappa} + 5(1-\muo)^{-1} K^{2}) \alpha\eta$, we get that
\[
    \| u_t + \alpha N_{\lambdao}(\theta_0) \| \leq K (1 - \etao \lambdao / (1 - \muo))^{t} + \lambdao^{-1} K^{2} \left( (3\overline{\kappa} + 5 K(1 - \muo)^{-1} (t \eta)\alpha^{2} + 9\alpha^{2} \right)
\]
Let us take $t$ that satisfies $(1 - \eta \lambda)^{t} = \alpha^{2}$, so that the two terms have approximately the same order, which is $ t_0 = \frac{2(1 - \muo)}{\etao \lambdao} \left \lceil \log \frac{1}{\alpha} \right\rceil$. For sufficiently small $\alpha$, the whole range $[t_0, 2t_0]$ lies within $ c_0 / (\eta \alpha)$, so that the bounds from Lemma~\ref{lemma_start_Ut_Dt} hold. This yields for some $C_0 = C_0(K, \kappa, \lambda, \mu)$ and sufficiently small $\alpha, \eta$
\[
    \max_{t \in [t_0, 2t_0]} \| u_{t} + \alpha H_{\lambda + \overline{\kappa}}(\theta_{t})^{-1} \nabla L(\theta_{t}) \| \leq C_0 \alpha^{2} \log \frac{1}{\alpha} \,.
\]
We will assume that the choice of $C_0$ also ensures this inequality for $t_0 + 1$ in place of $t_0$.

\paragraph{Stage 2: maintaining NG approximation.} Next, we show that we can maintain the approximation of the natural gradient by $u_t$ after step $t \geq t_0$ as well. We are going to use a stronger inequality to carry through by induction.

We first slightly refine the bound on $u_{t} - u_{t-1}$ for $t = t_0, t_0 + 1$. Since $\|\theta_t - \theta_{t-1}\| \leq A_t \leq C_1 \eta \alpha$, we get
\[
    \| u_{t} - u_{t-1}\| \leq 2C_0\alpha^{2} \log(1/\alpha) + \alpha^{2}\eta.
\]
In total, we say that for $t = t_0, t_0 + 1$,
\[
    \left\{ \| u_{t} + \alpha H_{\lambda + \overline{\kappa}}(\theta_{t})^{-1} \nabla L(\theta_{t}) \|^{2} + \| u_{t} - u_{t-1}\|^{2} \right\}^{1/2} \leq C_3 \alpha^{2} \log\left(\frac{1}{\alpha}\right)
\]
for $C_3 = 3C_1$ and sufficiently small $\alpha, \eta$. Furthermore, by Lemma~\ref{lemma_start_Ut_Dt} and Eq. \eqref{theta_update_via_u} we have for all $t \in [t_0, 2t_0]$,
\[
    \| \theta_{t} - \theta_{t-1} \| \leq C_3 \eta \alpha,
\]
as long as $C_3 \geq 3K + 5(1 - \mu)^{-1}$.

We are going to show by induction that
\begin{align}
    \max_{t_0 \leq t' \leq t} \left(\| u_{t'} + \alpha N_{\lambdao}(\theta_{t'})\|^{2} + \| u_{t'} - u_{t'-1}\|^{2} \right)^{1/2} &\leq C_3 \alpha^2 \log\left(\frac{1}{\alpha}\right), \label{big_bracket_induction} \\
    \max_{t_0 \leq t' \leq t} \| \theta_{t'} - \theta_{t'-1} \| & \leq C_3 \eta \alpha, \label{theta_local_diff_induction}
\end{align}
We know that it holds at $t = t_0 + 1$. Given it holds for $t$, we need to extend it to $t + 1$.
Then from Eq. \ref{u_eq_lissa_like_mom},
\[
    u_{t + 1} = (\mu + 1) u_{t} - \mu u_{t-1} - \etao \left[\alpha \nabla L(\theta_t) + H(\theta_t)u_{t}\right] + O_{\leq}(9K^{3} \eta \alpha^{2})
\]
Using the fact that $ \nabla L(\theta) = H_{\lambdao}(\theta_t) N_{\lambdao}(\theta_t)$ by definition, we obtain the recursion
\[
    u_{t + 1} + \alpha N_{\lambdao}(\theta_t) = (I - \etao H(\theta_t)_{\lambdao}(\theta_t))(u_{t} + \alpha N_{\lambdao}(\theta_t)) + \mu(u_{t} - u_{t-1}) + O_{\leq}(9K^{3} \eta \alpha^{2}),
\]
which can also be rewritten in the form
\[
    \begin{pmatrix}
         u_{t + 1} + \alpha N_{\lambdao}(\theta_t) \\
         u_{t + 1} - u_{t}
    \end{pmatrix}
    =
    \begin{pmatrix}
        I - \etao H_{\lambdao}(\theta_t) & \mu \\
         - \etao H_{\lambdao}(\theta_t) & \mu
    \end{pmatrix}
    \begin{pmatrix}
         u_{t} + \alpha N_{\lambdao}(\theta_t) \\
         u_{t} - u_{t-1}
    \end{pmatrix}
    + O_{\leq}(9K^{3} \eta \alpha^{2})
\]
From the proof of Lemma~\ref{lissa_mom}, we know that the norm of the matrix in the middle is bounded by $ 1 - \frac{1-\mu}{\lambda \eta} $. Therefore,
\begin{equation}\label{total_bound_with_theta_t}
    \left(\| u_{t+ 1} + \alpha N_{\lambdao}(\theta_{t})\|^{2} + \| u_{t+1} - u_{t}\|^{2} \right)^{1/2} \leq \left(1 - \frac{\lambda \eta}{1-\mu}\right) C_3\alpha^{2} \log\left(\frac{1}{\alpha} \right) + O_{\leq}(9K^{3} \eta \alpha^{2})
\end{equation}
To replace $\theta_{t}$ with $\theta_{t + 1}$ in the LHS, we need to bound the difference $\theta_{t+1} - \theta_t$, since by the regularity conditions and the triangle inequality we have
\begin{multline*}
    \left(\| u_{t+ 1} + \alpha N_{\lambdao}(\theta_{t + 1})\|^{2} + \| u_{t+1} - u_{t}\|^{2} \right)^{1/2} \\
    \leq \left(1 - \frac{\lambda \eta}{1-\mu}\right) C_3\alpha^{2} \log\left(\frac{1}{\alpha} \right) + 9K^{3} \eta \alpha^{2} + K \alpha \| \theta_{t + 1} - \theta_{t} \|
\end{multline*}
We first observe that from \eqref{total_bound_with_theta_t}, the second term does not compensate $ - C_3 \frac{\lambda \eta}{1-\mu} \alpha^{2} \log\left(\frac{1}{\alpha}\right) + 9K^3 \eta \alpha^{2} \leq 0$ (for small enough $\alpha$), so that by the regularity condition,
\begin{equation}\label{final_ut_bound}
    \| u_{t + 1} \| \leq K \alpha + C_3 \alpha^{2} \log\left(\frac{1}{\alpha} \right) \leq 3K \alpha
\end{equation}
We use \eqref{rnd_eq_2} to bound $\theta_{t + 1} - \theta_{t}$. Firstly, notice that given the approximation of $u_{t}$ by $N_{\lambdao}(\theta_t)$, we have that
\[
    \alpha \nabla L(\theta_t) + H_{\lambda}(\theta_t) u_{t} = -(\lambdao - \lambda) u_{t} + \alpha \nabla L(\theta_t) + H_{\lambda}(\theta_t) u_{t} = O_{\leq}(3(1 - \mu)\kappao K\alpha + C_3 \alpha^{2} \log(1/\alpha))
\]
Then, we get from \eqref{rnd_eq_2},
\[
    \| \theta_{t + 1} - \theta_{t}\| \leq \mu C_3 \eta \alpha + 3(1 - \mu)\kappao K\alpha + C_3 \alpha^{2} \log(1/\alpha) + 9K^{3} \eta \alpha^{2} \leq C_3 \eta\alpha,
\]
where we assume that $C_3 \geq 6 \kappao K$ and $\alpha $ is sufficiently small. Hence, the induction assumption on the difference $\| \theta_{t + 1} - \theta_{t}\|$ stands. We finalize the induction step by plugging this bound into the the inequality above,
\begin{align*}
    \left(\| u_{t+ 1} + \alpha N_{\lambdao}(\theta_{t + 1})\|^{2} + \| u_{t+1} - u_{t}\|^{2} \right)^{1/2} &
    \leq \left(1 - \frac{\lambda \eta}{1-\mu}\right) C_3\alpha^{2} \log\left(\frac{1}{\alpha} \right) + 9K^{3} \eta \alpha^{2} + C_3 K \eta\alpha^{2}
    \\
    &= C_3 \alpha^{2} \left((1 - \frac{\lambda\eta}{1-\mu})\log(1/\alpha) + (9K^{3} + C_3K)\eta\right) \\
    & \leq C_3 \alpha^{2} \log\left(\frac{1}{\alpha}\right),
\end{align*}
where the latter follows from the fact that $ \frac{\lambda\eta}{1 - \mu} \log\left(\frac{1}{\alpha}\right) \geq (9K^{3} + C_3K)\eta $ for small enough $\alpha$.

Hence, we have shown by induction that \eqref{big_bracket_induction}, \eqref{theta_local_diff_induction} hold for the rest of the training trajectory.

\paragraph{Connecting to NG descent.} Now we can finally show that updates mimic natural gradient descent,
\[
    \overline{\theta}_{t + 1} = \overline{\theta}_{t} - \gamma H_{\lambdao}(\overline{\theta}_{t})^{-1} \nabla L(\overline{\theta}_{t}), \qquad \overline{\theta}_{0} = \theta_{0},
\]
where recall that $\gamma = \kappao\eta \alpha$.

Using Eq. \eqref{theta_update_via_u} we have,
\begin{align*}
    \theta_{t + 1} - \overline{\theta}_{t + 1} &= {\theta}_{t}  + \kappao\eta u_{t} + (u_{t + 1} - u_{t}) - \overline{\theta}_{t} + \kappao\eta \alpha N_{\lambdao}(\overline{\theta}_{t}) \\
    & = {\theta}_{t} - \overline{\theta}_t + \kappao \eta (u_{t} + \alpha N_{\lambdao}(\overline{\theta}_t)) + (u_{t + 1} - u_{t}) \\
    & = {\theta}_{t} - \overline{\theta}_t + \kappao \eta (u_{t} + \alpha N_{\lambdao}({\theta}_t)) + \kappao \eta\alpha (N_{\lambdao}(\overline{\theta}_t) - N_{\lambdao}({\theta}_t)) +  (u_{t + 1} - u_{t})
\end{align*}
Denoting $\Delta_t = \| \theta_{t} - \overline{\theta}_t - u_{t}\|$, we first obtain by the regularity condition and Eq. \eqref{final_ut_bound},
\[
    \| N_{\lambdao}(\overline{\theta}_t) - N_{\lambdao}({\theta}_t) \| \leq K \| {\theta}_t - \overline{\theta}_t \| \leq K (\Delta_t + 3K\alpha).
\]
Therefore, we get the recursive inequality,
\begin{align*}
    \Delta_{t + 1} &\leq \Delta_{t} + C_3 \kappao \eta\alpha^{2} \log\frac{1}{\alpha} + K \kappao \eta\alpha(\Delta_t + 3K\alpha) \\
    & \leq (1 + K\kappao\eta\alpha) \Delta_t + (C_3 \log(1/\alpha) + 3K^{2})\kappao \eta\alpha^{2} \\
    & \leq \dots \\
    & \leq (1 + K\kappao\eta\alpha)^{t + 1 - t_0} \Delta_{t_0} + (1 + K\kappao\eta\alpha)^{t + 1 - t_0} (C_3 \alpha \log(1/\alpha) + 3K^{2}\alpha )\kappao \eta\alpha (t + 1 - t_0)
\end{align*}
Given that the maximal number of steps $T$ satisfies $ \eta \alpha T \leq A $ (which corresponds to the finite sum of the step sizes of NG, and it is one of the conditions of the theorem), we have that
\[
    (1 + K\kappao\eta\alpha)^{t - t_0} \leq (1 + K\kappao\eta\alpha)^{T} \leq \exp(K\kappao A),
\]
bounded by a constant. Therefore, for all $t \leq T$,
\[
    \Delta_{t} \leq O(\Delta_{t_0} + \alpha \log(1/\alpha)).
\]
It is left to bound $\Delta_{t_0}$, which simply relies on the fact that $t_0$ steps is not enough to diverge too far from the initial point $\theta_0$. Indeed,
\[
    \| \overline{\theta}_{t_0} - \theta_0 \| \leq \gamma \sum_{t = 0}^{t_0-1} \| N_{\lambdao}(\overline{\theta}_{t})\| \leq K t_0 \gamma = O(\alpha \log(1/\alpha)),
\]
and similarly, recalling Lemma~\ref{lemma_start_Ut_Dt},
\[
    \| \overline{\theta}_{t_0} - \theta_0 \| \leq (t_0 - 1)A_t \leq t_0 (3K \kappao \eta \alpha + 5(1 - \muo)^{-1} K^{2} \eta\alpha) = O(\alpha\log(1/\alpha)).
\]
We also note that $\Delta_{t_0} \leq \| \overline{\theta}_{t_0} - \theta_0 \| + \| {\theta}_{t_0} - \theta_0 \| $. Then,
\[
    \| \theta_{t} - \overline{\theta}_t\| \leq \| u_t\| + 
    \Delta_t = O(\alpha) + O(\alpha\log(1/\alpha)) = O(\alpha\log(1/\alpha)).
\]



\end{document}